\documentclass{article}
\usepackage[english]{babel}
\usepackage[utf8]{inputenc}

\usepackage{graphicx}

\usepackage{natbib}

\usepackage{algorithm}
\usepackage[noend]{algpseudocode}
\algrenewcommand{\algorithmiccomment}[1]{$\vartriangleright$ #1}
\algrenewcommand{\algorithmicreturn}{\textbf{Return: }}
\algnewcommand\algorithmicinput{\textbf{Input: }}
\algnewcommand\Input{\State \algorithmicinput}

\usepackage[strict]{changepage}

\usepackage{bm}

\usepackage{amsmath}
\usepackage{amssymb}

\usepackage{amsthm}
\usepackage{mathtools}
\mathtoolsset{showonlyrefs}
\usepackage{mathbbol}

\usepackage{color}
\usepackage{xcolor}
\usepackage{enumitem}

\usepackage{xspace}
\usepackage{nicefrac}

\usepackage{booktabs}

\usepackage[super]{nth}

\usepackage{hyperref}

\usepackage[accepted]{icml2019}

\icmltitlerunning{Geometric Losses for Distributional Learning}

\newtheorem{lemma}{Lemma} 
\newtheorem{proposition}{Proposition}

\makeatletter
\addto\extrasenglish{%
  \renewcommand{\sectionautorefname}{\S\@gobble}%
  \renewcommand{\subsectionautorefname}{\S\@gobble}%
  \renewcommand{\subsubsectionautorefname}{\S\@gobble}%
}
\makeatother

\def\EE{{\mathbb E}}

\def\RR{{\mathbb R}}

\def\p{{\bm \p}}

\def\p{{\bm p}}

\newcommand{\Simplex}{\triangle}

\newcommand{\cM}{{\mathcal M}}
\newcommand{\cE}{{\mathcal E}}

\newcommand{\cD}{{\mathcal D}}
\newcommand{\cR}{{\mathcal R}}
\newcommand{\cC}{{\mathcal C}}
\newcommand{\cU}{{\mathcal U}}

\newcommand{\cX}{{\mathcal X}}
\newcommand{\cY}{{\mathcal Y}}

\newcommand{\cF}{{\mathcal F}}

\newcommand{\probspace}{{\mathcal{M}_1^+(\cY)}}
\newcommand{\posspace}{{\mathcal{M}^+(\cY)}}
\newcommand{\contspace}{{\mathcal{C}(\cY)}}

\DeclareMathOperator*{\argmin}{argmin}
\DeclareMathOperator*{\argmax}{argmax}

\DeclareMathOperator*{\supp}{supp}

\begin{document}
\twocolumn[
\icmltitle{Geometric Losses for Distributional Learning}
\icmlsetsymbol{equal}{*}

\begin{icmlauthorlist}
\icmlauthor{Arthur Mensch}{ens,cnrs}
\icmlauthor{Mathieu Blondel}{ntt}
\icmlauthor{Gabriel Peyré}{ens,cnrs}
\end{icmlauthorlist}

\icmlaffiliation{ntt}{NTT Communication Science Laboratories, Kyoto, Japan}
\icmlaffiliation{ens}{École Normale Supérieure, DMA, Paris, France}
\icmlaffiliation{cnrs}{CNRS, France}
\icmlcorrespondingauthor{AM}{arthur.mensch@m4x.org}
\icmlcorrespondingauthor{MB}{mathieu@mblondel.org}
\icmlcorrespondingauthor{GP}{gabriel.peyre@ens.fr}

\icmlkeywords{Optimal transport, distributional learning}
\vskip 0.3in
]

\printAffiliationsAndNotice

\begin{abstract}

Building upon recent advances in entropy-regularized optimal transport, and upon Fenchel duality between measures and continuous functions, we propose a generalization of the logistic loss that incorporates a metric or cost between classes. Unlike previous attempts to use optimal transport distances for learning, our loss results in unconstrained convex objective functions, supports infinite (or very large) class spaces, and naturally defines a geometric generalization of the softmax operator. The geometric properties of this loss make it suitable for predicting sparse and singular distributions, for instance supported on curves or hyper-surfaces. We study the theoretical properties of our loss and showcase its effectiveness on two applications: ordinal regression and drawing generation.
\end{abstract}


\vspace{-0.4cm}
\section{Introduction}

For probabilistic classification, the most popular loss is arguably the
(multinomial) logistic loss. It is smooth, enabling fast convergence rates, and
the softmax operator provides a consistent mapping to probability
distributions. 
In many applications, different costs are associated to misclassification errors between classes. 
While a cost-aware generalization of the logistic loss
exists \citep{margincrf}, it does not provide a cost-aware counterpart of the softmax.
The softmax is pointwise by nature: it is oblivious to misclassification
costs or to the geometry of classes. 

Optimal transport (Wasserstein) losses have recently gained popularity in machine learning, for their ability to 
compare probability distributions in a geometrically faithful manner, with applications such as classification~\citep{kusner2015word}, clustering~\cite{cuturi2014fast}, domain adaptation~\citep{courty2017joint}, dictionary learning~\cite{rolet2016fast} and generative models training~\citep{montavon2016wasserstein,arjovsky2017wasserstein}.
For probabilistic classification, \citet{frogner_learning_2015} proposes to use
entropy-regularized optimal transport \citep{cuturi_sinkhorn_2013} in the
multi-label setting. Although this approach successfully leverages a cost between classes,
it results in a non-convex loss, when combined with a softmax. A similar regularized Wasserstein loss is used by~\citet{luise2018wasserstein} in conjunction with a kernel ridge regression procedure~\cite{ciliberto2016consistent} in order to obtain a consistency result. 

The relation between the logistic loss and the maximum entropy principle is
well-known. Building upon a generalization of the Shannon entropy originating
from entropy regularized optimal transport \citep{feydy_interpolating_2018} and 
Fenchel duality between measures and continuous functions, we propose
a generalization of the logistic loss that takes into account a metric or
cost between classes. Unlike previous attempts to use optimal transport
distances for learning, our loss is convex, and naturally defines a
geometric generalization of the softmax operator. Besides providing novel
insights in the logistic loss, our loss is theoretically sound, even when
learning and predicting \textit{continuous} probability distributions over a potentially infinite number of classes. To sum
up, our contributions are as follows.

\paragraph{Organization and contributions.}

\begin{itemize}[topsep=0pt,itemsep=0pt,parsep=3pt,leftmargin=15pt]

\item We introduce the distribution learning setting, review existings
losses leveraging a cost between classes and point out their shortcomings
(\autoref{sec:related_work}).

\item Building upon entropy-regularized optimal transport, we present a novel
cost-sensitive distributional learning loss and its corresponding softmax operator.
Our proposal is theoretically sound even in continuous measure spaces
(\autoref{sec:construction}).

\item We study the theoretical properties of our loss, such as its Fisher
consistency (\autoref{sec:properties}). We derive tractable methods to compute and minimize it in the discrete distribution setting. We propose an abstract Frank-Wolfe scheme for computations in the continuous setting.

\item Finally, we demonstrate its effectiveness on two discrete prediction tasks involving a geometric cost: ordinal
regression and drawing generation using VAEs~(\autoref{sec:exps}).

\end{itemize}



\paragraph{Notation.} 

We denote $\cX$ a finite or infinite input space, and $\cY$ a compact
potentially infinite output space. When $\cY$ is a finite set of $d$ classes,
we write $\cY = [d] \triangleq \{1, \dots, d\}$. We denote $\contspace$,
$\cM(\cY)$, $\cM^+(\cY)$ and $\cM^+_1(\cY)$ the sets of continuous (bounded)
functions, Radon (positive) measures and probability measures on $\cY$. Note that in finite dimensions,
$\cM^+_1([d]) = \triangle^d$ is the probability simplex and $\cC([d]) = \RR^d$.
We write vectors in $\RR^d$ and continuous
functions in $\contspace$ with small normal letters, e.g., $f, g$. 
In the finite
setting, where $\cY = \{y_1, \dots, y_d\}$, we define $f_i \triangleq f(y_i)$.
We write elements of $\Simplex^d$ and measures in $\cM(\cY)$ with greek letters
$\alpha, \beta$. We write matrices and operators with capital letters, e.g., $C$.
We denote by $\otimes$ and $\oplus$ the tensor product and sum, and $\langle \cdot,
\cdot \rangle$ the scalar product.

\section{Background}
\label{sec:related_work}

In this section, after introducing distributional learning in a discrete
setting, we review two lines of work for taking into account a cost $C$ between
classes: cost-augmented losses, and geometric losses based on Wasserstein and
energy distances. Their shortcomings motivate the introduction of a new
geometric loss in \autoref{sec:construction}.

\subsection{Discrete distribution prediction and learning}
\label{sec:discrete}

We consider a general predictive setting in which an input vector~$x \in \cX$ is
fed to a parametrized model $g_\theta : \cX \to \RR^d$ (e.g., a neural network),
that predicts a score vector $f = g_\theta(x) \in \RR^d$. At test time, that
vector is used to predict the most likely class
$\hat y = \argmax_{y \in [d]} f_y$.
In order to predict a probability distribution $\alpha \in \triangle^d$, it is
common to compose $g_\theta$ with a link function $\psi(f)$,
where $\psi \colon \RR^d \to \triangle^d$.
A typical example of link function is the softmax.

To learn the model parameters $\theta$, it is necessary to define a loss
$\ell(\alpha, f)$ between a ground-truth $\alpha \in \triangle^d$ and the score
vector $f \in \RR^d$.  Composite losses \citep{reid_composite_binary,
vernet_2016} decompose that loss into a loss $\ell_\triangle(\alpha, \beta)$,
where $\ell_\triangle \colon \triangle^d \times \triangle^d \to \RR$
and $\psi$:
$\ell(\alpha, f) \triangleq \ell_\triangle(\alpha, \psi(f))$.
Note that depending on $\ell_\triangle$ and $\psi$, $\ell$ is not necessarily
convex in $f$. More recently, \citet{blondel_learning_2018,fy_losses_journal} introduced
Fenchel-Young losses, a generic way to directly construct a loss $\ell$ and a
corresponding link $\psi$. We will revisit and generalize that framework to the
continuous output setting in the
sequel of this paper. Given a loss $\ell$ and a training set of
input-distribution pairs, $(x_i, \alpha_i)$, where $x_i \in \cX$ and $\alpha_i
\in \triangle^d$, we then minimize
$\sum_i \ell(\alpha_i, g_\theta(x_i))$,
potentially with regularization on $\theta$.

\subsection{Cost-augmented losses}\label{sec:cost-augmented}

Before introducing a new geometric cost-sensitive loss in
\autoref{sec:construction}, let us now review classical existing
cost-sensitive loss functions.  Let $C$ be a $d \times d$ matrix, such that
$c_{y,y'}\ge 0$ 
is the cost of misclassifying class $y \in [d]$ as class $y' \in [d]$. We
assume $c_{y,y} = 0$ for all $y \in [d]$. To take into account the cost $C$,
in the single label setting, it is natural to define
a loss $L \colon [d] \times \RR^d \to \RR$ as follows
\begin{equation}
    L(y, f) = c_{y,y'} \quad \text{where} \quad y' \in \argmax_{i \in [d]}
    {f}_{i}.
\label{eq:cost_sensitive_non_convex}
\end{equation}
To obtain a loss $\ell \colon \triangle^d \times \RR^d \to \RR$, we simply
define $\ell(\delta_y, f) \triangleq L(y, f)$,
where $\delta_y$ is the one-hot representation of $y \in [d]$.
Note that choosing $c_{y,y'}=1$ when $y \neq y'$ and $c_{y,y'}=0$ otherwise
(i.e., $C = 1 - I_{d \times d}$) reduces to the zero-one loss.
To obtain a convex upper-bound, \eqref{eq:cost_sensitive_non_convex} is
typically replaced with a cost-augmented hinge loss
\citep{multiclass_svm,structured_hinge}:
\begin{equation}
    L(y, f) = \max_{i \in [d]} ~ c_{y,i} + f_i - f_y.
\end{equation}
Replacing the max above with a log-sum-exp leads to a cost-augmented version of
the logistic (or conditional random field) loss \citep{margincrf}.
Another convex relaxation is the cost-sensitive \textit{pairwise} hinge loss
\citep{multiclass_weston,duchi_multiclass_2016}.
Remarkably, all these losses use only one row of $C$, the one corresponding to
the ground truth $y$. Because of this dependency on $y$, it is not clear how to
define a probabilistic mapping at \textit{test} time. In this paper, we propose a loss
which comes with a geometric generalization of the softmax operator. That
operator uses the entire cost matrix $C$.

\subsection{Wasserstein and energy distance losses}


Wasserstein or optimal transport distances recently gained popularity as a loss
in machine learning for their ability to compare probability distributions in a
geometrically faithful manner. As a representative application,
\citet{frogner_learning_2015} proposed to use entropy-regularized optimal
transport \citep{cuturi_sinkhorn_2013} for cost-sensitive multi-label
classification. Effectively, optimal transport lifts a distance or cost
$C \colon \cY \times \cY \to \RR_+$ to a distance between probability
distributions over $\cY$. Following \citet{genevay_stochastic_2016}, given a
ground-truth probability distribution $\alpha \in \triangle^d$ and a predicted
probability distribution $\beta \in \triangle^d$, we define
\begin{equation}\label{eq:primal_transport}
    \text{OT}_{C, \epsilon}(\alpha, \beta) \triangleq
    \min_{\pi \in \cU(\alpha, \beta)} \langle \pi, C \rangle 
    + \varepsilon \text{KL}(\pi \vert \alpha \otimes \beta),
\end{equation}
where $\cU$ is the transportation polytope, a subset of $\Simplex^{d \times d}$
whose elements $\pi$ have constrained marginals: $\pi 1 = \alpha$ and $\pi^\top
1 = \beta$. $\text{KL}$ is the Kullback--Leibler divergence (a.k.a. relative
entropy).  Because $\beta$ needs to be a valid probability distribution,
\citet{frogner_learning_2015} propose to use $\beta = \psi(f) = \text{softmax}(f)$, where
$f \in \RR^d$ is a vector of prediction scores. Unfortunately, the resulting
composite loss, $\ell(\alpha,f) = \text{OT}_{C,
\epsilon}(\alpha,\text{softmax}(f))$, is not convex w.r.t. $f$. 
Another class of divergences between measures $\alpha$ and $\beta$ stems from
energy distances~\cite{szekely_energy_2013} and maximum mean discrepancies. However, composing these
divergences with a softmax again breaks convexity in $f$.
In contrast, our proposal is convex in $f$ and defines a natural geometric
softmax.

\section{Continuous and cost-sensitive distributional learning and
prediction}\label{sec:construction}

In this section, we construct a loss between probability measures and
score functions, canonically associated with a link function.  Our construction
takes into account a cost function $C : \cY \times \cY \to \RR$ between classes.
Unlike existing methods reviewed in~\autoref{sec:cost-augmented}, our loss
is well defined and convex on compact, possibly infinite spaces~$\cY$. We start
by extending the setting of \autoref{sec:discrete} to predicting
\textit{arbitrary probabilities}, for instance having continuous densities with
respect to the Lebesgue measure or singular distributions supported on curves or
surfaces.


\subsection{Continuous probabilities and score functions}

We consider a compact metric space of outputs $\cY$, endowed with a \textit{symmetric} cost
function $C: \cY \times \cY \to \RR$. We wish to predict probabilities over
$\cY$, that is, learn to predict distributions 
$\alpha \in \probspace$. The
space of probability measures forms a closed subset of the space of Radon
measures $\cM(\cY)$, i.e., $\probspace \subseteq \cM(\cY)$. From the Riesz
representation theorem, $\cM(\cY)$ is the topological dual of the space
of continuous measures $\contspace$, endowed with the uniform convergence norm
$\Vert \cdot \Vert_\infty$. 
The topological duality between the primal $\cM(\cY)$ and the dual $\contspace$
defines a pairing, similar to a ``scalar product'', between these
spaces:
\begin{equation} 
\langle \alpha,f  \rangle \triangleq \int_\cY f(y) \textrm{d}\alpha(y) 
= \EE[f(Y)],
\end{equation} 
for all $\alpha \in \cM(\cY)$ and $f \in \contspace$, where $Y$ is a random variable with law $\alpha$. 
This pairing also defines the natural topology to compare measures and to
differentiate functionals depending on measures. This is the so-called
weak$^\star$ topology, which corresponds to the convergence in law of random
variables. A sequence $\alpha_n$ is said to converge weak$^\star$ to some
$\alpha$ if for all functions $f \in \mathcal{C}(\mathcal{Y})$, $\langle
\alpha_n,f  \rangle  \rightarrow \langle \alpha,f  \rangle$. 
Note that when endowing $\cM(\cY)$ with this weak$^\star$ topology, the dual of
$\cM(\cY)$ is $\contspace$, which is the key to be able to use duality (and in
particular Legendre-Fenchel transform) from convex optimization.
Using this topology is fundamental to define geometric losses that can cope with
arbitrary, possibly highly localized or even singular distributions (for
instance sparse sums of Diracs or measures concentrated on thin sets such as 2-D
curves or 3-D surfaces). 

Similarly to the discrete setting reviewed in \autoref{sec:discrete}, in the
continuous setting, we now wish to predict a distribution $\alpha \in
\probspace$ by setting $\alpha = \psi(f)$, where $f = g_\theta(x) \in
\contspace$,
$g_\theta \colon \cX \to \contspace$ (i.e., $g_\theta$ is unconstrained), 
and $\psi \colon
\contspace \to \probspace$ is a link function.
We propose to use maps between the primal $\probspace$ and the \textit{dual}
score space $\contspace$ as link functions. As we shall see, such
\textit{mirror} maps are naturally defined by continuous convex function on the
primal space, through Fenchel-Legendre duality. Our framework
recovers the discrete case $\cY = [d]$ as a particular case, with 
$\Simplex^d$ corresponding to $\cM_1^+([d])$
and $\RR^d$ to $\cC([d])$,
though the isomorphisms $\alpha \to \sum_{i=1}^d \alpha_i \delta_i$ and for all $i \in[d]$, $f(i) = f_i$.

Regularization of optimal transport is our key tool to construct entropy functions
which are continuous with respect to the weak$^\star$ topology, and that can be
conjugated to define a $\contspace \to \probspace$ link function.
It allows us to naturally leverage a cost $C: \cY \times \cY
\to \RR$ between classes.

\subsection{An entropy function for continuous probabilities}\label{sec:sinkhorn_entropy}

The regularized optimal transport cost~\eqref{eq:primal_transport} remains well
defined when $\alpha$ and $\beta$ belong to a continuous measure space
$\probspace$, with $\cU$ now being a subset of $\cM_1^+(\cY \times \cY)$ with marginal
constraints. It induces the self-transport
functional~\cite{feydy_interpolating_2018}, that we reuse for our purpose:
\begin{equation}
\label{eq:Sinkhorn_negentropy}
    \Omega_C(\alpha) \triangleq \left\{ \begin{array}{lr}
        - \frac{1}{2} \text{OT}_{C, \varepsilon = 2}(\alpha, \alpha)
        &\:\text{for}\: \alpha {\in} \cM^+_1(\cY) 
        \\
        +\infty&\quad\text{otherwise}.
    \end{array}\right.
\end{equation}
We will omit the dependency of $\Omega$ on $C$ when clear from context.  It is shown
by~\citet{feydy_interpolating_2018} that $\Omega$ is continuous and convex on
$\cM(\cY)$, and strictly convex on $\cM^+_1(\cY)$, where continuity is taken
w.r.t. the weak$^\star$ topology. We call $\Omega$,
the \textit{Sinkhorn negentropy}. As a negative entropy function, it can be used
to measure the uncertainty in a probability distribution (lower is more
uncertain), as illustrated in \autoref{fig:2d}.
It will prove crucial in our loss construction. In
the above, we have set w.l.o.g. $\varepsilon = 2$ to recover simple asymptotical behavior of $\Omega$, as will be clear in \autoref{prop:asymptotic}.

We first recall some known results from~\citet{feydy_interpolating_2018}. Using
Fenchel-Rockafellar duality theorem~\cite{rockafellar_extension_1966}, the
function $\Omega$ rewrites as the solution to a Kantorovich-type dual
problem~\citep[see e.g.,][]{villani_optimal_2008}. For all $\alpha \in
\probspace$, we have that
\begin{align}
    -\Omega_C(\alpha) &=
    \max_{f \in \contspace} \langle \alpha, f \rangle - \log \langle \alpha \otimes \alpha, e^{\frac{f \oplus f - C}{2}} \rangle \label{eq:dual_kantorovich}, 
\end{align}    
where we use the \textit{homogeneous} dual (i.e. with a log in the
maximization), as explained in~\citet{cuturi_semidual_2018}.

\paragraph{Gradient and extrapolation.}
$\Omega$ is differentiable in the sense of
measures~\cite{santambrogio_optimal_2015}, meaning that there exists a
continuous function $\nabla \Omega(\alpha)$ such that, for all $\xi_1$, $\xi_2
\in \probspace$, $t > 0$,
\begin{equation}\label{eq:displacement}
    \Omega(\alpha + t(\xi_2 - \xi_1)) = \Omega(\alpha) + t \langle\xi_2 - \xi_1, \nabla \Omega(\alpha)\rangle + o(t).
\end{equation}
As shown in \citet{feydy_interpolating_2018}, this function $f=\nabla \Omega(\alpha)$, that we call the \textit{symmetric Sinkhorn potential}, is a particular solution of the dual problem. It is the only function in $\contspace$ such that $-f = T(-f, \alpha)$, where
the soft $C$-transform operator \cite{cuturi_semidual_2018} is defined as
\begin{equation}
    T(f, \alpha)(y) \triangleq - 2 \log \langle \alpha, e^{\frac{f - C(y,
    \cdot)}{2}} \rangle.
\end{equation}
This operator can be understood as the log-convolution of the measure $\alpha e^{\frac{f}{2}}$ with the Sinkhorn kernel $e^{-\frac{C}{2}}$.
The Sinkhorn potential $f$ has the remarkable property of being defined on all $\cY$, even though the support of $\alpha$ may be smaller.  Given any dual solution $g$ to~\eqref{eq:dual_kantorovich}, which is defined $\alpha$-almost everywhere, we have $f = -T(-g, \alpha)$, i.e. $f$ \textit{extrapolates} the values of $g$ on the support of $\alpha$, using the Sinkhorn kernel.

\paragraph{Special cases.}

The following proposition, which is an original contribution, shows that the
Sinkhorn negentropy asymptotically recovers the negative Shannon entropy and
Gini index \citep{gini_index} when rescaling the cost.
The Sinkhorn negentropy therefore defines a
parametric family of negentropies, recovering these important special cases.
Note however that on continuous spaces $\mathcal{Y}$, the Shannon entropy is not
weak$^\star$ continuous and thus cannot be used to define geometric loss and link
functions, the softmax link function being geometry-oblivious. Similarly, the Gini index is not defined on $\probspace$, as it involves the \textit{squared values} of $\alpha$ in a discrete setting.
\begin{proposition}[Asymptotics of Sinkhorn negentropies]\label{prop:asymptotic}
    For $\cY$ compact, the rescaled Sinkhorn negentropy
    converges to a kernel norm for high regularization $\varepsilon$. 
    Namely, for all $\alpha \in \probspace$, we have
    \begin{align}
        \varepsilon \Omega_{C/\varepsilon}(\alpha) &
        \overset{\varepsilon \to +\infty}{\longrightarrow}      
        \frac{1}{2} \langle \alpha \otimes \alpha, -C \rangle.
    \end{align}
    Let $\cY = [d]$ be discrete and choose $C = 1
    - I_{d \times d}$. The Sinkhorn negentropy converges to the
    Shannon negentropy for low-regularization, and into the negative Gini index for
    high regularization:
    \begin{equation}
        \Omega_{C/\varepsilon}(\alpha) 
        \overset{\varepsilon \to 0}{\longrightarrow} 
        \langle \alpha, \log \alpha \rangle, \:
        \varepsilon \Omega_{C/\varepsilon}(\alpha) 
        \overset{\varepsilon \to +\infty}{\longrightarrow} 
        \frac{1}{2} (\Vert \alpha \Vert_2^2 - 1).
    \end{equation}
\end{proposition}
Proof is provided in \autoref{app:asymptotic}. The first part of the proposition shows that the Sinkhorn negentropies
converge to a kernel norm~\citep[see e.g.,][]{sriperumbudur_universality_2011}. This is similar to the
regularized Sinkhorn divergences converging to an Energy
Distance~\cite{szekely_energy_2013} for $\varepsilon \to \infty$
\citep{genevay_sinkhorn-autodiff_2017,feydy_interpolating_2018}.

\begin{figure}
    \centering
    \includegraphics[width=\linewidth]{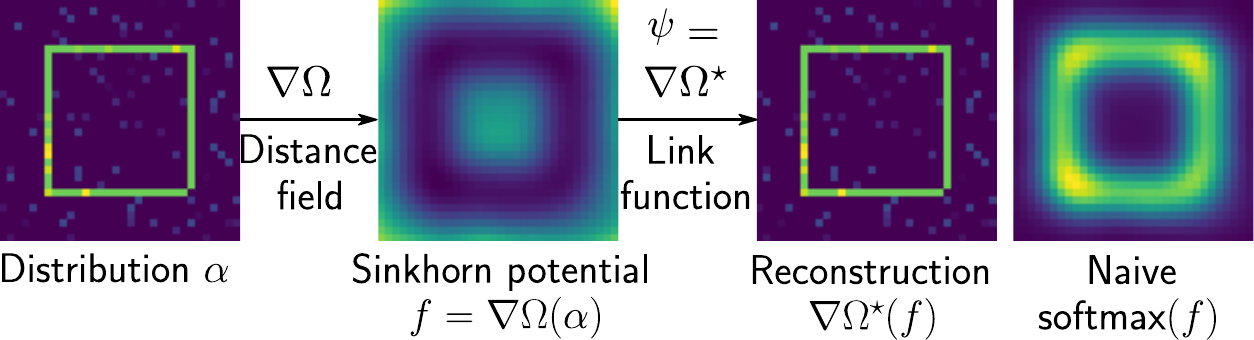}
    \caption{The symmetric Sinkhorn potentials form a distance field to a
        weighted measure. The link function $\psi = \nabla \Omega^\star(f)$
allows to go back from this field in $\contspace$ to a measure $\alpha \in
\probspace$.
}
\label{fig:distance}
\end{figure}

\paragraph{From probabilities to potentials.}

The symmetric Sinkhorn potential $f = \nabla \Omega(\alpha)$ is a continuous
function, or a vector in the discrete setting. It can be interpreted as a
distance field to the distribution~$\alpha$. We visualize this field on a 2D space in
\autoref{fig:distance}, where $\cY$ is the set of $h \times w$ pixels of an image, and we
wish to predict a 2-dimensional probability distribution in $\probspace =
\Simplex^{h \times w}$. Predicting a distance field $f \in \contspace$ to a
measure is more convenient than predicting a distribution directly, as it has
unconstrained values and is therefore easier to optimize against.
For this reason, we propose to learn parametric models that predict a ``distance
field" $f = g_\theta(x)$ given an input $x \in \cX$. In the following section, 
we construct 
a link function $\psi: \contspace \to \probspace$, for general probability
measure and function spaces $\probspace$ and $\contspace$, so to obtain a distributional estimator $\alpha_\theta = \psi \circ g_\theta: \cX \to \probspace$.
\subsection{Fenchel-Young losses in continuous setting}

To that end,
we generalize in this section the recently-proposed Fenchel-Young (FY) loss framework
\citep{blondel_learning_2018,fy_losses_journal}, originally limited to discrete cost-oblivious measure spaces, to
infinite measure spaces. Inspired by that line of work, we use Legendre-Fenchel duality
to define loss and link functions from Sinkhorn negative entropies, in a principled manner.
We define the Legendre-Fenchel conjugate $\Omega^\star \colon \contspace \to
\RR$ of $\Omega$ as
\begin{equation}\label{eq:fenchel-transform}
    \Omega^\star(f) \triangleq \max_{\alpha \in \probspace} 
    \langle \alpha, f \rangle - \Omega(\alpha).
\end{equation}
Rigorously, $\Omega^\star(f)$ is a
pre-conjugate, as $\Omega$ is defined on $\cM(\cY)$, the topological dual of
continuous functions $\cC(\cY)$. For a comprehensive and rigorous treatment of
the theory of conjugation in infinite spaces, and in particular Banach spaces
as is the case of $\cC(\cY)$, see \citet{mitter_convex_2008}.

As $\Omega$ is strictly convex, $\Omega^\star$ is differentiable everywhere and
we have, from a Danskin theorem \cite{danskin_theory_1966} with left Banach
space and right compact space~\citep[Theorem C.1]{bernhard_variations_1990}:
\begin{equation*}
    \nabla \Omega^\star(f) \triangleq \argmax_{\alpha \in \probspace} 
    \langle \alpha, f \rangle - \Omega(\alpha) \in \contspace.
\end{equation*}
That gradient can be used as a link $\psi$ from $f \in \contspace$ to $\alpha
\in \probspace$.  It can also be interpreted as a regularized prediction
function \citep{blondel_learning_2018, mensch_differentiable_2018}.
Following the FY loss framework, we define the loss associated with
$\nabla \Omega^\star$ by
\begin{equation}\label{eq:fy_loss}
    \ell_\Omega(\alpha, f) \triangleq \Omega^\star(f) + \Omega(\alpha) - 
    \langle \alpha, f \rangle.
\end{equation}
In the discrete single-label setting, that loss is also related to
the construction of \citet[Proposition 3]{duchi_multiclass_2016}.
From the Fenchel-Young theorem \cite{rockafellar_convex_1970},
$\ell_\Omega(\alpha, f) \ge 0$,
with equality if and only if $\alpha = \nabla \Omega^\star(f)$. 
The loss $\ell_\Omega$ is
thus positive, convex and differentiable in its second argument, and minimizing it amounts to find
the pre-image $f^\star$ of the target distribution $\alpha$ with respect to the
link (mapping) $\nabla \Omega^\star$. 

Our construction is a generalization of the Fenchel-Young loss framework
\citep{blondel_learning_2018,fy_losses_journal}, in the sense that
it relies on topological duality between $\contspace$ and
$\probspace$, instead of the Hilbertian structure of $\RR^d$ and $\Simplex^d$,
to construct the loss $\ell_\Omega$ and link function $\nabla \Omega^\star$.
We now instantiate the Fenchel-Young loss \eqref{eq:fy_loss} 
with Sinkhorn negentropies in order to obtain a novel cost-sensitive loss.

\subsection{A new geometrical loss and softmax}

The key ingredients to derive a Fenchel-Young loss $\ell_\Omega$ and a link
$\nabla \Omega^\star$ are the conjugate $\Omega^\star$ and its gradient.
Remarkably, they enjoy a simple form with Sinkhorn negentropies, as shown in the following 
proposition. 

\begin{proposition}[Conjugate of the Sinkhorn negentropy]\label{prop:sinkhorn}
    For all $f \in \cC(\cY)$,
    the Legendre-Fenchel conjugate $\Omega^\star$ of $\Omega$ defined in
    \eqref{eq:Sinkhorn_negentropy} and its gradient read
    \normalfont
    \begin{align}
        \text{g-LSE}(f) &\triangleq \Omega^\star(f) = - \log \min_{\alpha \in \probspace} \Phi(\alpha, f) \\
        \text{g-softmax}(f) &\triangleq \nabla \Omega^\star(f) = \argmin_{\alpha \in \probspace} \Phi(\alpha, f)\\
        \text{where} \quad \Phi(\alpha, f) &\triangleq \langle \alpha \otimes \alpha, \exp(-\frac{f \oplus f + C}{2})\rangle\label{eq:quadratic_main}
    \end{align}
    \textit{and where g stands for geometric and LSE for log-sum-exp.}
\end{proposition}
The proof can be found in~\autoref{app:prop_sinkhorn}. 
$\nabla \Omega^\star(f)$ is the usual
Fréchet derivative of $\Omega^\star$, that lies a priori in the topological dual
space of $\contspace$, i.e. $\cM(\cY)$. From a Danskin
theorem~\citep{bernhard_variations_1990}, it is in fact a probability measure. The
probability distribution $\alpha = \nabla \Omega^\star(f)$ is typically \textit{sparse},
as the minimizer of a quadratic on a convex subspace of $\cM(\cY)$. 
We call the loss $\ell_\Omega$ generated by the Sinkhorn negentropy
\textbf{g-logistic} loss.

\paragraph{Special cases.}

Let $\cY = [d]$ and $C = 1 - I_{d \times d}$ ($0$-$1$ cost matrix).  From
\autoref{prop:asymptotic}, $\Omega$ asymptotically recovers the negative Shannon
entropy when $\Omega= \Omega_{\frac{C}{\varepsilon}}$ as $\varepsilon \to 0$ and the
negative Gini index when $\Omega = \varepsilon \Omega_{\frac{C}{\varepsilon}}$, as
$\varepsilon \to \infty$. $\nabla \Omega^*$ is then equal to 
$\text{softmax}(f) \triangleq \frac{\exp f}{\sum_i
\exp f_i}$, and to $\text{sparsemax}(f) \triangleq \argmin_{\alpha \in \triangle^d}
\|\alpha - f\|^2$ \citep{martins_softmax_2016}, respectively.  Likewise,
$\ell_\Omega$ recovers the logistic and sparsemax losses. When 
$\varepsilon \to 0$, because
$(\frac{C}{\varepsilon})_{y, y'} = \infty$ if $y \neq y'$ and $0$ otherwise, we see
that the logistic loss infinitely penalizes inter-class errors. That is, to
obtain zero logistic loss, the model must assign probability $1$ to the correct
class. The limit case $\varepsilon \to 0$ is the only case for which g-softmax
always outputs completely dense distributions. In the continuous case,
 $\varepsilon \Omega^\star_{C / \varepsilon}(f / \varepsilon)$
degenerates into a positive deconvolution objective with simplex constraint:
\begin{equation}
\max_{\alpha \in \probspace} \langle \alpha, f \rangle - \frac{1}{2} \langle \alpha \otimes \alpha, -C \rangle.
\end{equation}
Fig. \ref{fig:distance} shows that $\nabla \Omega^\star$ has indeed a deconvolutional effect.

\begin{figure*}
    \includegraphics[width=\textwidth]{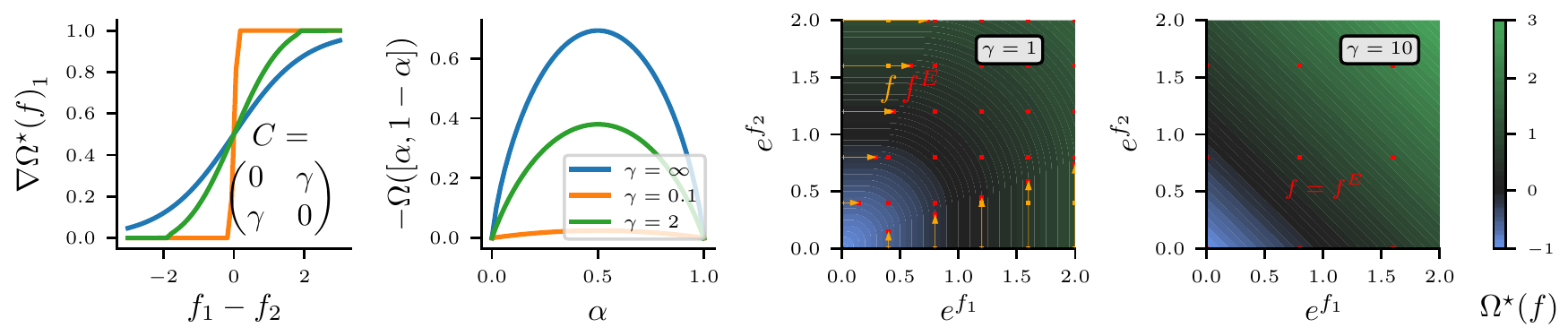}
    \vspace*{-2em}
    \caption{Left: Geometric softmax and Sinkhorn entropy, for 
    symmetric cost matrices, in the binary case. Predictions from
the g-softmax are sparse, as the minimizer of a convex quadratic on the simplex. Right: Level sets of 
    the geometric conjugate. Introducing a cost matrix induces a deformation $\Simplex^2$, the level-set of the log-sum-exp operator, onto the set of symmetric Sinkhorn potentials $\cF$. The geometric conjugate defines an extrapolation operator $f \to f^E$ that replaces the score function onto the cylinder $\cF + \RR 1$.}\label{fig:2d}
    \vspace*{-1em}
\end{figure*}

\subsection{Computation}\label{sec:computation}

Before studying the g-logistic loss $\ell_\Omega$ and link function 
$\nabla \Omega^\star(f)$, we now describe practical algorithms for computing
 $\nabla \Omega^\star(f)$ and $\Omega^\star(f)$ in the discrete and continuous cases. The
key element in using the g-LSE as a layer in an arbitrary complex model is to
minimize the quadratic function $\Phi_f \triangleq \Phi(\cdot, f)$, on $\probspace$. We can then use the minimum value in the forward pass, and the minimizer in the backward pass, during e.g. SGD training.

\paragraph{Continuous optimisation.} In the general case where $\cY$ is compact, we cannot represent $\alpha \in \probspace$ using a finite vector. Yet, we can use a Frank-Wolfe scheme to progressively add spikes, i.e. Diracs to an iterate sequence $\alpha_t$. For this, we need to compute, at each iteration, the gradient of $\Phi_f$ in the sense of measure~\eqref{eq:displacement}, i.e. the function in $\contspace$
\begin{equation}
    \nabla \Phi_f(\alpha) = \exp(-\frac{f + T(-f, \alpha)}{2}),
\end{equation}
that simply requires to compute the C-transform of $-f$ on the measure $\alpha$, similarly to regularized optimal transport. The simplest Frank-Wolfe scheme then updates
\begin{equation}
    y_t \in \argmin_{\cY} \nabla \Phi_f(\alpha_{t-1}), \:
    \alpha_{t} = \alpha_{t-1} + \frac{2}{t+2} (\delta_{y_t} - \alpha_{t-1}).
\end{equation}
Indeed, for $h \in \contspace$, the minimizer of $\langle h, \cdot \rangle$ on $\probspace$ is the Dirac $\delta_y$ where $y \in \cY$ minimizes $h$. This optimization scheme may be refined to ensure a geometric convergence
of $\Phi_f(\alpha_t)$. It can be used to identify Diracs from a continuous distance field $f$, similar to super-resolution approaches proposed in~\citet{bredies2013inverse,boyd2017alternating}. 
It requires to work with computer-friendly representation of $f$, so that we can obtain an approximation of $y_t$ efficiently, using e.g. non-convex optimization. Another approach is to rely on a deep parametrization of a particle swarm, as proposed by \citet{boyd2018deeploco}. We leave such an application for future work, and focus on an efficient \textit{discrete} solver for the g-LSE and g-softmax.

\paragraph{Discrete optimisation.}In the discrete case, we can parametrize 
$\log \Phi(f, \cdot)$ in logarithmic space, by
setting $\alpha = \exp(l) - \text{LSE}(l)$, with $l \in \RR^d$. 
$\Omega^\star(f)$ then reads
\begin{equation}\label{eq:non_convex}
    \max_{l \in \RR^d} - \log \sum_{i, j=1}^d e^{l_i + l_j -\frac{f_i + f_j + c_{i, j}}{2}} + 2 \text{LSE}(l).
\end{equation}
This objective is non-convex on $\RR^d$ but invariant with translation and convex on 
$\{h \in \RR^d, \text{LSE}(l) = 0\}$.
 It thus admits a unique solution, that we can find using an unconstrained quasi-Newton solver like L-BFGS~\cite{liu_limited_1989}, that we stop when the iterates are sufficiently stable.
 For $l$ that maximizes~\eqref{eq:non_convex}, the gradient $\nabla \Omega^\star(f) = \text{softmax}(l)$ is used for backpropagation and at test time. As $\nabla \Omega^\star(f)$ is sparse, we expect some coordinates $l_i$ to go to $-\infty$. In practice, $\alpha_i$ then underflows to $0$ after a few iterations.

\paragraph{Two-dimensional convolution.}In the discrete case, when dealing with two-dimensional potentials and measures, the objective function~\eqref{eq:non_convex}
can be written with a convolution operator, as 
$- \log \langle e^{l - \frac{f}{2}}, e^{-\frac{C}{2}} \star e^{l - \frac{f}{2}} \rangle$
where $C \in \RR^{(h \times w)^2}$. It is therefore efficiently computable and differentiable on GPUs, especially when the kernel $C$ is separable in height and width, e.g. for the $\ell_2^2$ norm,
in which case we perform 2 successive one-dimensional
convolutions. We use this computational trick in our variational auto-encoder experiments (\autoref{sec:exps}).

\section{Geometric and statistical properties}\label{sec:properties}

We start by studying the mirror map $\nabla \Omega^\star$, that we expect to invert the mapping $\alpha \to \nabla \Omega(\alpha)$. This study is necessary as we cannot rely on typical conjugate inversion results~\citep[e.g.,][Theorem 26.5]{rockafellar_convex_1970}, that would stipulate that $\nabla \Omega^\star = (\nabla \Omega)^{-1}$ on the domain of $\Omega^\star$. Indeed, this result is stated in finite dimension, and requires that $\Omega$ and $\Omega^\star$ be Legendre, i.e. be strictly convex and differentiable on their domain of definition, and have diverging derivative on the boundaries of these domains~\citep[see also][]{wainwright_graphical_2008}. This is not the case of the Sinkhorn negentropy, which requires novel adjustements.
With these at hands, we show that parametric models involving a final g-softmax layer can be trained to minimize a certain well-behaved Bregman divergence
on the space of probability measures. Proofs are reported in~\autoref{app:geometric} and~\autoref{app:statistic}

\subsection{Geometry of the link function}

We have constructed the link function $\nabla \Omega^\star$ in hope that it
would allow to go
from a symmetric Sinkhorn potential $f = \nabla \Omega(\alpha)$ back to the original
measure $\alpha$. The following lemma states that this is indeed the case,
and derives two consequences on the space of symmetric Sinkhorn potentials, defined
as $\cF \triangleq \{ f \in \contspace, f = \nabla \Omega(\alpha) \}$.

\begin{lemma}[Inversion of the Sinkhorn potentials]\label{lemma:inversion}
    \begin{align}
        &\forall\,\alpha \in \probspace,\:\nabla \Omega^\star \circ \nabla \Omega(\alpha) = \alpha. \\
        &\forall\,f \in \cF, \quad \nabla \Omega \circ \nabla \Omega^\star(f) = f,
        \quad \Omega^\star(f) = 0.
    \end{align}
\end{lemma}

The computation of the Sinkhorn potential thus inverts the g-LSE operator on the
space $\cF$, which is included in the $0$-level set of $\Omega^\star$. This is
similar to the set $\cF_{\text{Shannon}} = \{ \log \alpha, \alpha \in \Simplex^d
\}$ being the $0$ level set of the log-sum-exp function when using the Shannon
negentropy as $\Omega$.

This corollary is not sufficient for our purpose, as we want to characterize the action of
$\nabla \Omega^\star$ on \textit{all continuous functions} $f \in \contspace$. For this, note that the g-LSE operator $\Omega^\star$ has the same behavior as the log-sum-exp when composed with the addition of a constant $c \in \RR$:
\begin{equation}\label{eq:translation}
    \Omega^\star(f + c) = \Omega^\star(f) + c,\quad \nabla \Omega^\star(f + c) = \nabla \Omega^\star(f).
\end{equation}
Therefore, for all $f \in \contspace$, $\Omega^\star(f- \Omega^\star(f)) = 0$,
which almost makes $f - \Omega^\star(f)$ a part of the space of potentials~$\cF$. Yet, in contrast with the Shannon entropy case, the inclusion of $(\Omega^\star)^{-1}({0})$ in $\cF$ is strict. Indeed, following~\autoref{sec:sinkhorn_entropy}
$f \in \cF$ implies that there exists $\alpha \in \probspace$ such that $f = -T(-f, \alpha)$ is the image of the C-transform operator. The operator $\nabla \Omega \circ \nabla \Omega^\star$ has therefore an \textit{extrapolation} effect, as it replaces $f - \Omega^\star(f)$ onto the set of Sinhorn potentials. This is made clear by the following proposition.

\begin{proposition}[Extrapolation effect of $\nabla \Omega \circ \nabla \Omega^\star$]%
    \label{prop:extrapolation}
    For all $f \in \contspace$, we define the extrapolation of $f$ to be 
    \begin{equation}
        f^E \triangleq -T \big( - ( f - \Omega^\star(f)), \nabla \Omega^\star(f)\big) + \Omega^\star(f).
    \end{equation}
    Then, for all $f \in \contspace$, $\nabla \Omega \circ \nabla \Omega^\star(f) = f^E - \Omega^\star(f).$
\end{proposition}
The extrapolation operator translates $f$ to $(\Omega^\star)^{-1}({0})$,
extrapolates $f - \Omega^\star(f)$ so that it becomes a Sinkhorn potential, then
translates back the result so that $\Omega^\star(f^E) = \Omega^\star(f)$. Its effects clearly appears on \autoref{fig:2d} (right), where we see that $f^E$ is a projection of $f$ onto the cylinder $\cF + \RR$.

\subsection{Relation to Hausdorff divergence}

Recall that the Bregman divergence \citep{bregman_1967} generated by a strictly
convex $\Omega$ is defined as
\begin{equation}
    D_\Omega(\alpha, \beta) \triangleq
    \Omega(\alpha) - \Omega(\beta) - \langle \nabla \Omega(f), \alpha - \beta).
\end{equation}
When $\Omega$ is the classical negative Shannon entropy $\Omega(\alpha) =
\langle \alpha, \log \alpha \rangle$, it is well-known that $D_\Omega$ equals
the Kullback-Leibler divergence and it is easy to check that
\begin{equation}
    \ell_\Omega(\alpha, f) = D_\Omega(\alpha, \nabla \Omega^\star(f)) =
    \text{KL}(\alpha, \text{softmax}(f)).
\end{equation}
The equivalence between Fenchel-Young loss $\ell_\Omega(\alpha, f)$ and
composite Bregman divergence $D_\Omega(\alpha, \nabla \Omega^\star(f))$,
however, no longer holds true when $\Omega$ is the Sinkhorn negentropy defined
in \eqref{eq:Sinkhorn_negentropy}. In that case, $D_\Omega$ can be
interpreted as an \textit{asymmetric} Hausdorff divergence
\citep{aspert2002mesh,feydy_interpolating_2018}. It forms a geometric divergence akin to OT distances, and estimates the distance between distribution supports. 
As we now show, $\ell_\Omega$ provides an upper-bound on
the composition of that divergence with $\nabla \Omega^\star$.
\begin{proposition}[$\ell_\Omega$ upper-bounds Hausdorff divergence]\label{prop:upper_bound}
    \begin{align}
        &D_\Omega(\alpha, \nabla \Omega^\star(f)) = \ell_\Omega( \alpha, f^E) \\
            &= \ell_\Omega( \alpha, f) - \langle \alpha,
            f^E - f\rangle \leq \ell_\Omega( \alpha, f)
    \end{align}
    with equality if $\text{\upshape supp}\, \nabla \Omega^\star(f) = \text{\upshape supp}\, \alpha$.
\end{proposition}

In contrast with the KL divergence, the asymmetric Hausdorff divergence is finite even 
when $\supp \alpha \neq \supp \beta$, a geometrical property that it shares with optimal transport divergences.
We now use \autoref{prop:upper_bound} to derive a new consistency result
justifying our loss.  Let
us assume that input features and output distributions follow a distribution 
$\cD
\in \cM^+_1(\cX \times \cM^+_1(\cY))$. We define the Hausdorff divergence risk
and the Fenchel-Young loss risk as
\begin{equation*}
    \cE(\beta) \triangleq \EE[D_\Omega(\alpha,
    \beta(x))]
    \quad \text{and} \quad
    \cR(g) \triangleq \EE[\ell_\Omega( \alpha, g(x))],
\end{equation*}
where the expectation is taken w.r.t. $(x,\alpha) \sim \cD$.
We define their associated Bayes estimators as
\begin{equation*}
    \beta^\star \triangleq \argmin_{\beta \colon \cX \to \probspace}
    \cE(\beta) 
    \quad \text{and} \quad
    g^\star \triangleq \argmin_{g \colon \cX \to \contspace}
    \cR(g).
\end{equation*}
The next proposition guarantees calibration of $\ell_\Omega$ with respect to the
asymmetric Hausdorff divergence $D_\Omega$. 

\begin{proposition}[Calibration of the g-logistic loss]
\label{prop:bregman} 

The g-logistic loss $\ell_\Omega$ where $\Omega$ is defined in
\eqref{eq:Sinkhorn_negentropy} is Fisher consistent with
the Hausdorff divergence $D_\Omega$ for the same $\Omega$.
That is,
\begin{align}
    \cE(\beta^\star) = \cR(g^\star),
    \quad\text{with}\quad 
    g^\star = \nabla \Omega \circ \beta^\star.
\end{align}
The excess of risk in the Hausdorff divergence is controlled by
the excess of risk in the g-logistic loss. For all 
$g: \cX \to \contspace$, we have
\begin{equation}
    \cE(\nabla \Omega^\star \circ g) - \cE(\beta^\star) \leq
    \cR(g) - \cR(g^\star).
\end{equation}
\end{proposition}
This result, that follows the terminology of~\citet{tewari_consistency_2005}, shows that $\ell_\Omega$ is suitable for learning predictors that minimize $D_\Omega$.

\section{Applications}\label{sec:exps}

We present two experiments that demonstrate the validity and usability
of the geometric softmax in practical use-cases. We provide a PyTorch package for reusing the discrete geometric softmax layer\footnote{\url{github.com/arthurmensch/g-softmax}}.

\subsection{Ordinal regression}

We first demonstrate the g-softmax for
\textit{ordinal regression}. In this setting, we wish to predict an ordered
category among $d$ categories, and we assume that the cost of predicting $\hat y$
 instead of $y$ is symmetric and depends on the difference between $\hat y$ and
$y$. For instance, when predicting ratings, we may have three categories
\textit{bad $\prec$ average $\prec$ good}. This is typically modeled by a
cost-function $C(\hat y, y) = \phi(| \hat y - y |)$, where $\phi$ is the
$\ell_2^2$ or $\ell_1$ cost.
We use the real-world ordinal datasets provided by
\citet{gutierrez_ordinal_2016}, using their predefined 30 cross-validation
folds.

\paragraph{Experiment and results.} 

We study the performance of the geometric
softmax in this discrete setting, where the score function is assumed to be a
linear function of the input features $x \in \RR^k$, i.e, $g_{W, b}(x) = W x +
b$, with $b \in \RR^d$, $x \in \RR^k$ and $W \in \RR^{d \times k}$. We compare
its performance to multinomial regression, and to immediate threshold and
all-threshold logistic regression~\cite{rennie_loss_2005}, using a reference
implementation provided by \citet{pedregosa_consistency_2017}. We use a
cross-validated $\ell_2$ penalty term on the linear score model~$g_\theta$. To
compute the Hausdorff divergence at test time and the geometric loss during
training, we set $C(\hat y, y) = (\hat y - y)^2 / 2$.

\begin{table}
    \caption{Performance of geometric loss as a drop-in replacement in linear
        models for ordinal regression. Our method performs better w.r.t.\ its natural
    metric, the Hausdorff divergence.}
    \vspace{.5em}

    \label{table:ordinal}
    {\footnotesize
    \begin{tabular}{lllll}
        \toprule
        {} &             LR &         LR(AT) &         LR(IT) &       \textbf{g-logistic} \\
        \midrule
        Haus. div. &  $.46{\pm}.12$ &  $.47{\pm}.14$ &  $.59{\pm}.16$ &  $\mathbf{.44{\pm}.08}$ \\
        MAE         &  $.44{\pm}.09$ &  $\mathbf{.42{\pm}.06}$ &  $.44{\pm}.08$ &  $.45{\pm}.09$ \\
        Acc.        &  $\mathbf{.66{\pm}.07}$ &  $.65{\pm}.06$ &  $.65{\pm}.06$ &  $.65{\pm}.07$ \\
        \bottomrule
    \end{tabular}
    }
\end{table}

The results, aggregated over datasets and cross-validation folds, are reported
in \autoref{table:ordinal}. We observe that
the g-logistic regression
performs better than the others for the Hausdorff divergence on average.
It performs slightly
worse than a simple logistic regression in term of accuracy, but slightly better
in term of mean absolute error (MAE, the reference metric in ordinal
regression). It thus provides a viable alternative to thresholding techniques,
that performs worse in accuracy but better in MAE. It has the further advantage
of naturally providing a distribution of output given an input $x$. We simply
have, for all $y \in [d]$, $p(Y=y |\,X=x) = {(\text{g-softmax}(g_{W,
b}(x)))}_{y}$. 

\paragraph{Calibration of the geometric loss.}

We validate \autoref{prop:bregman} experimentally on the ordinal
regression dataset \textit{car}. During training, we measure the geometric
cross-entropy loss and the Hausdorff divergence on the train and validation set.
\autoref{fig:training} shows that $\ell_\Omega$ is indeed an upper bound of
$D_\Omega$, and that the difference between both terms reduces to almost $0$ on
the train set. \autoref{prop:bregman} ensures this finding provided that
the set of scoring function is large enough, which appears to be approximately
the case here.

\begin{figure}[ht]
    \begin{minipage}[c]{0.5\linewidth}
        \includegraphics[width=\textwidth]{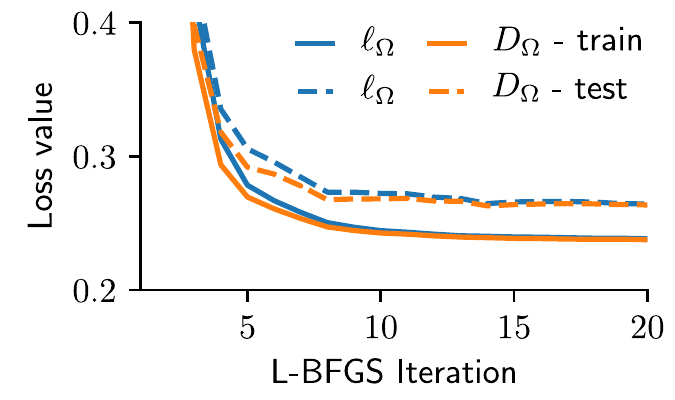}
    \end{minipage}\hfill
    \begin{minipage}[c]{0.45\linewidth}
        \vspace{-1em}
    \caption{Training curves for ordinal regression on dataset \textit{car}. The
    difference between the g-logistic loss and the Hausdorff divergence vanishes on the
train set.}\label{fig:training}
    \end{minipage}
    \vspace{-1em}
\end{figure}

\subsection{Drawing generation with variational auto-encoders}

The proposed geometric loss and softmax are suitable to estimate distributions
from inputs. As a proof-of-concept experiment, we therefore focus on a setting
in which distributional output is natural: generation of hand-drawn doodles and
digits, using the Google QuickDraw \cite{ha_neural_2017} and MNIST dataset. We
train variational autoencoders on these datasets using, as output layers, (1)
the KL divergence with normalized output and (2) our geometric
loss with normalized output.  These approaches output an image prediction using
a softmax/g-softmax over all pixels, which is justified when we seek to
output a concentrated distributional output. This is the case for doodles and
digits, which can be seen as 1D distributions in a 2D space. It differs from the
more common approach that uses a binary cross-entropy loss for every pixel and
enables to capture interactions between pixels at the feature extraction level. We use standard KL penalty on the latent space distribution.

Using the g-softmax takes into account a cost between pixels $(i, j)$ and
$(k, l)$, that we set to be the Euclidean cost $C/\sigma$, where $C$
is the $\ell_2^2$ cost and $\sigma$ is the typical distance of interaction---we choose $\sigma=2$ in our experiments.
We therefore made the hypothesis that it would help in reconstructing the input
distributions, forming a non-linear layer that captures interaction between
inputs in a non-parametric way.

\paragraph{Results.}We fit a simple MLP VAE on 28x28 images from the QuickDraw Cat dataset.
Experimental details are reported in \autoref{app:exp_details} (see
\autoref{fig:big_vae}). We also present an experiment with 64x64 images and a
DCGAN architecture, as well as visualization of a VAE fitted on MNIST. In \autoref{fig:vae}, we compare the reconstruction and the samples after training our model with the g-softmax and simple softmax loss. Using the g-softmax,
which has a deconvolutional effect, yields images that are concentrated near the
edges we want to reconstruct. We compare the training curves for both the softmax and
g-softmax version: using the g-softmax link function and its associated loss better minimizes the asymmetric Hausdorff divergence. The cost of computation is again increased by a factor 10.

\begin{figure}
    \includegraphics[width=\linewidth]{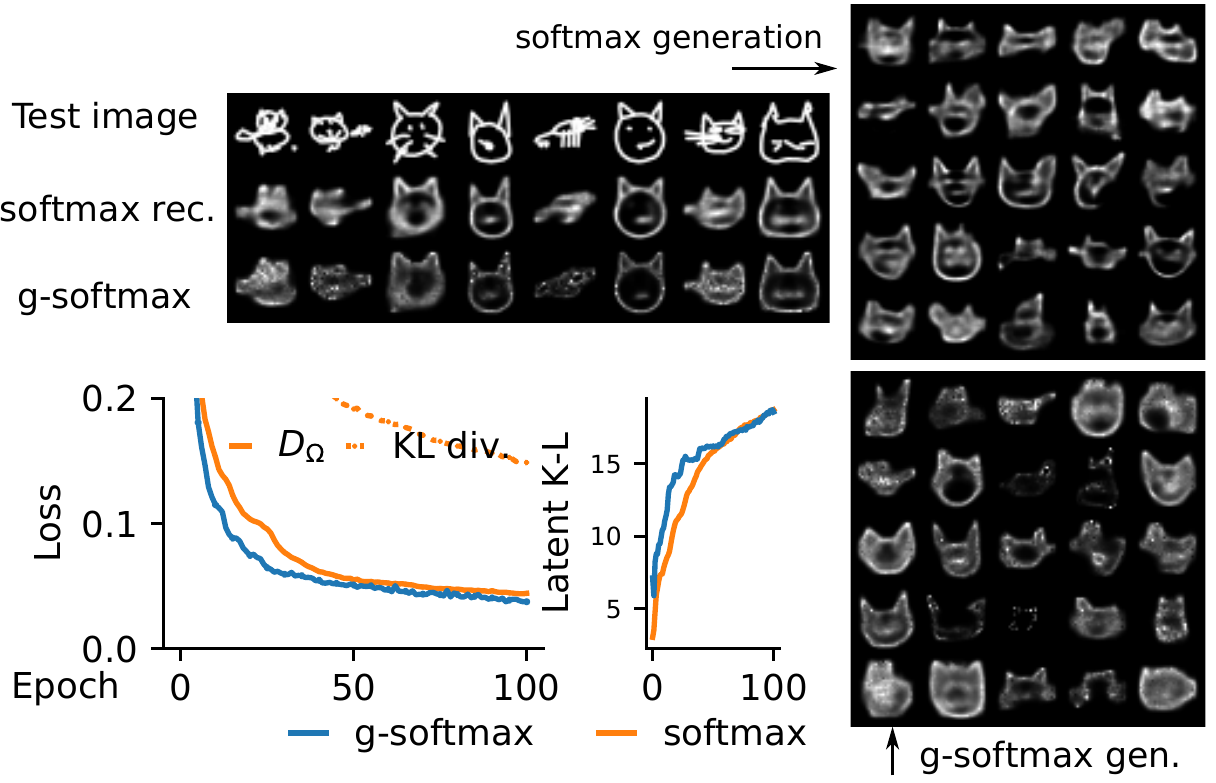}
    \caption{The g-softmax layer permits to generate and reconstruct drawing in a more concentrated manner. For a same level of variational penalty, the g-softmax better
    and faster minimizes the asymmetric Hausdorff divergence. See also \autoref{fig:big_vae}.}\label{fig:vae}
    \vspace*{-1em}
\end{figure}

\section{Conclusion}

We introduced a principled way of learning distributional predictors in
potentially continuous output spaces, taking into account a cost function in
between inputs. We constructed a geometric softmax layer, that we derived from
Fenchel conjugation theory in Banach spaces. The key to our construction is an entropy function derived from regularized optimal transport, convex and
weak$^\star$ continuous on probability measures. Beyond the experiments in discrete measure spaces that we presented, our framework opens the
doors for new applications that are intrinsically off-the-grid, such as
super-resolution.

\clearpage
\section*{Acknowledgements}
The work of A. Mensch and G. Peyré has
been supported by the European Research Council
(ERC project NORIA). A. Mensch thanks Jean Feydy and Thibault Séjourné for fruitful discussions.

\bibliographystyle{icml2019}
\bibliography{biblio}

\clearpage

\newlength{\offsetpage}
\setlength{\offsetpage}{1.0cm}

\changetext{}{-2\offsetpage}{-\offsetpage}{\offsetpage}{}
\renewcommand{\headwidth}{\textwidth}

\onecolumn\vspace*{-.25in}
{\center
{\Large\bf Appendix}\bottomtitlebar}
\appendix

\section{Proofs}

We prove propositions by order of appearance in the main text.

\subsection{Asymptotics of the Sinkhorn negentropy---Proof of \autoref{prop:asymptotic}}\label{app:asymptotic}

\begin{proof}
    We start by showing the Shannon entropy limit of the Sinkhorn entropy, in the discrete case. In this case, we use the standard Kantorovich dual~\cite{cuturi_sinkhorn_2013}. Let $\varepsilon > 0$, $\alpha \in \Simplex^d$, and
    \begin{equation}\label{eq:kantorovich_sym}
        \Omega(\alpha) \triangleq \Omega_{C / \varepsilon}(\alpha) = - \max_{f \in \RR^d}
        \langle \alpha, f \rangle - \langle \alpha \otimes \alpha, \exp(\frac{f \oplus f - C}{2}) \rangle + 1.
    \end{equation}
    For all $f \in \RR^d$
    \begin{equation}
        \Psi_\alpha(f) \triangleq\langle \alpha, f \rangle - \langle \alpha \otimes \alpha, \exp(\frac{f \oplus f - C / \varepsilon}{2}) + 1 \rangle
        = \sum_{i=1}^d f_i \alpha_i - \sum_{i, j= 1}^d \alpha_i \alpha_j
        \exp(\frac{f_i + f_j - c_{i, j} / \varepsilon}{2}) + 1.
    \end{equation}
    For $f$ optimal in \eqref{eq:kantorovich_sym}, letting $\varepsilon \to 0$, we have, using element-wise multiplication $\ast$,
    \begin{equation}
        \nabla \Psi_\alpha(f) = \alpha - \alpha^2 \ast e^f = 0\quad\text{i.e.}\quad
        e^{f_i} = \frac{1}{\alpha_i}\quad\text{for all }i \in [d].
    \end{equation}
    Replacing in \eqref{eq:kantorovich_sym}, we obtain
    \begin{equation}
        \Omega(\alpha) = \langle \alpha, \log \alpha \rangle + \sum_{i=1}^d \alpha_i - 1 = \langle \alpha, \log \alpha \rangle.
    \end{equation}
    Let us now consider the limit for $\varepsilon \to \infty$ of $\Omega_{C / \varepsilon}(\alpha)$, for an arbitrary symmetric cost matrix $C$. We rewrite
    \begin{equation}
        \Omega_{C / \varepsilon}(\alpha) = 
        \max_{f \in \contspace} 2 \langle \alpha, \frac{f}{2} \rangle - \varepsilon
        \langle \alpha \otimes \alpha, e^{\frac{\frac{f \oplus f}{2}- C}{\varepsilon}}\rangle = \text{OT}_{\varepsilon}(\alpha, \alpha).
    \end{equation}
    The asymptotic behavior of $\varepsilon \Omega_{C / \varepsilon}(\alpha)$, namely
        \begin{align}
            \varepsilon \Omega_{C/\varepsilon}(\alpha) &
            \overset{\varepsilon \to +\infty}{\longrightarrow}      
            \frac{1}{2} \langle \alpha \otimes \alpha, -C \rangle,
        \end{align}
    is then a simple consequence of the asymptotics of Sinkhorn OT distances~\cite{genevay_sinkhorn-autodiff_2017}, that we apply in the symmetric case. In the discrete setting, the result for $\varepsilon \to \infty$ becomes, if $C = 1 - I_{d \times d}$,
    \begin{equation}
        \frac{1}{2} \langle \alpha \otimes \alpha, I_{d \times d} - 1 \rangle = \frac{1}{2} \sum_{i=1}^d \alpha_i^2 - 1,
    \end{equation}
    as $\langle \alpha \otimes \alpha, 1 \rangle = 1$, which concludes the proof.

\end{proof}

\subsection{Construction of the geometric softmax---Proof of \autoref{prop:sinkhorn}}\label{app:prop_sinkhorn}

\begin{proof}

We can rewrite the self transport with the change of variable $\mu = \alpha e^{\frac{f}{2}} \in \posspace$,  due to \citet{feydy_global_2018}. We then have $\frac{f}{2} =  -\log \frac{\text{d}\alpha}{\text{d}\mu}$, and  
\begin{align*}
    \Omega(\alpha)\triangleq - \frac{1}{2} \text{OT}_2(\alpha, \alpha) &= - \max_{f \in \contspace}
     \langle \alpha, f \rangle 
     - \log \langle
    \alpha \otimes \alpha, \frac{\exp(f \oplus f - C)}{2} \rangle \\
    &= - \max_{\mu \in \posspace} - 2 \langle \alpha, 
    \log \frac{\text{d}\alpha}{\text{d}\mu} \rangle - \log {\Vert \mu \Vert}_{k_2}^2,
    \\\text{where}\quad{\Vert \mu \Vert}_{k_2} &\triangleq \int_\cX \int_\cX \exp(\frac{-C(x, y)}{2})\text{d}\mu(x)\text{d}\mu(y)
\end{align*}
is the kernel norm defined with kernel $k_2 \triangleq e^{-\frac{C}{2}}$.
Then, the conjugate of $\Omega(\alpha)$ reads, for all $f \in \contspace$,
\begin{align*}
    \Omega^\star(f) &= \max_{\alpha \in \probspace} \langle \alpha, f \rangle - \Omega(\alpha) \\
    &= \max_{\substack{\alpha \in \probspace \\ 
    \mu \in \posspace}}
    \langle \alpha, f \rangle - 2 \langle \alpha, 
    \log \frac{\text{d}\alpha}{\text{d}\mu} \rangle - \log {\Vert \mu \Vert}_{k_2}^2 \\
    &= \max_{\mu \in \posspace} \log \frac{\iint_{\cX^2} \exp\frac{(f(x) +f(y)}{2})\text{d}\mu(x)\text{d}\mu(y)}{\iint_{\cX^2} \exp(-\frac{C(x, y)}{2})\text{d}\mu(x)\text{d}\mu(y)},
\end{align*}
where we have used the conjugation of the relative entropy over the space of probability measure $\probspace$:
\begin{equation*}
    \max_{\alpha \in \probspace}
    \langle \alpha, f \rangle - 2\langle \alpha, 
    \log \frac{\text{d}\alpha}{\text{d}\mu} \rangle = 2 \log \int_\cX \exp(\frac{f(x)}{2}) \text{d} \mu(x).
\end{equation*}
We now revert the first change of variable, setting $\beta = \mu e^{\frac{f}{2}} \in \posspace$, and $\alpha = \frac{\beta}{\int_\cX \text{d}\nu} \in \probspace$. We have
\begin{align*}
    \Omega^\star(f) &= \max_{\alpha \in \probspace} - \log \iint_{\cX^2}
    \exp(-\frac{f(x) + f(y) + C(x, y)}{2})\text{d}\alpha(x)\text{d}\alpha(y),
\end{align*}
and the first part of the proposition follows:
\begin{align*}
    \text{g-LSE}(f) = \Omega^\star(f) &= - \min_{\alpha \in \probspace}
     \langle \alpha \otimes \alpha, \exp(- \frac{f \oplus f + C}{2}) \rangle.
\end{align*}

We have assumed that $\exp(-\frac{C}{2})$ is positive definite, which ensures that the bivariate function
\begin{equation}\label{eq:quadratic}
    \Phi(f, \alpha) \triangleq
    \langle \alpha \otimes \alpha, \exp(-\frac{f \oplus f + C}{2}) \rangle
\end{equation}
is strictly convex in $\alpha$ and in $f$. Let $\alpha^\star \triangleq \argmin_{\alpha \in \probspace} \Phi(f, \alpha)$. The gradient of $\Phi$ with respect to $f$
is a measure that reads
\begin{align}
    \nabla_f \Phi(f, \alpha)
    &= -\alpha \,\exp(-f - T_C(-f, \alpha)) \in \cM(\cY),\quad\text{where we recall} \\
    T_C(f, \alpha) &\triangleq - 2 \log \langle \alpha, \exp(\frac{f - C}{2}) \rangle.
\end{align}
From a generalized version of the Danskin theorem~\cite{bernhard_variations_1990}, the function
\begin{equation}
    f \to \argmin_{\alpha \in \probspace}  \langle \alpha \otimes \alpha, \exp(-\frac{f \oplus f + C}{2}) \rangle
\end{equation}
is differentiable everywhere and has for gradient
$\nabla_f \Phi(f, \alpha^\star)$. 
Composing with the $\log$, we obtain
\begin{equation}
    \nabla \Omega^\star(f) \in \probspace,\quad\text{and}\quad\nabla \Omega^\star(f) \propto \alpha^\star \,\exp(-f - T_C(-f, \alpha^\star)),
\end{equation}
where $\propto$ indicates proportionality. To conclude, we use \autoref{lemma:first_order}, that describes the minimizers of~\eqref{eq:quadratic}, and that we prove in the next section. It ensures that $-f - T_C(-f, \alpha^\star) = 0$ on the support of~$\alpha^\star$. Therefore
\begin{equation}
    \nabla \Omega^\star(f) = \alpha^\star \in \probspace,
\end{equation}
and the proposition follows.
\end{proof}

\subsection{Geometry of the link function---Proofs of \autoref{lemma:inversion} and \autoref{prop:extrapolation}}\label{app:geometric}

We first state and proof \autoref{lemma:first_order} on optimality condition in the minimization of $\alpha \to \Phi(\alpha, f)$. We then prove \autoref{lemma:inversion}, establish some basic properties of the extrapolation operator and prove \autoref{prop:extrapolation}.

\subsubsection{Necessary and sufficient condition of optimality in $\nabla \Omega^\star(f)$}\label{app:lemma_first_order}

Finding the minimizer $\alpha$ of $\alpha \to \Phi(\alpha, f)$ amounts to finding
the distribution for which $-f$ and its C-transform $T(-f, \alpha)$ are the less distant, as it appears in the following lemma.

\begin{lemma}[$\nabla \Omega^\star$ from first order optimality condition]\label{lemma:first_order}
    $\nabla \Omega^\star(f)$ is the only distribution $\alpha \in \probspace$ such
    that there exists a constant $A \in \RR$ such that
    \begin{equation}\label{eq:first_order}
        \begin{split}
            \frac{f(y) + T(-f, \alpha)(y)}{2} = A \quad\forall\, y \in \supp \alpha \\ 
            \frac{f(y) + T(-f, \alpha)(y)}{2} \leq A \quad\forall\, y \in \cY / \supp \alpha,
        \end{split}
    \end{equation}
    We then have $A = 2 \Omega^\star(f)$. \eqref{eq:first_order} form sufficient optimality conditions for finding $\nabla \Omega^\star(f) = \alpha$.
\end{lemma}

\begin{proof}We use an infinite version of the KKT condition~\citep[Section 9]{luenberger1997optimization} to solve the optimality of $\phi$, as defined in \eqref{eq:quadratic}. We fix $f \in \contspace$. The Lagrangian associated to the minimization of $\alpha \to \phi(f, \alpha)$ over the space of probability measure $\cM(\cX)$ reads
\begin{equation}
    L(\alpha, \mu, \nu)
    \triangleq \Phi(f, \alpha) + \langle \alpha, \mu \rangle + \nu (\langle \alpha, 1 \rangle - 1).
\end{equation}
A necessary and sufficient condition for $\alpha^\star$ to be optimal is the existence of a function $\mu \in \contspace$ and a real $\nu \in \RR$ such that,
\begin{align}
    \alpha^\star &\in \probspace\quad\text{(primal feasibility)},\\
    \forall\,y\in \cY,\quad- \nabla_\alpha \Phi(f, \alpha^\star)(y) &= \mu(y) + \nu\quad\text{(stationarity)},\\
    \forall\,y\in \cY,\quad\mu(y) &\leq 0\quad\text{(dual feasibility)}, \\
    \forall\,y \in \supp(\alpha^\star),\quad\mu(y) &= 0\quad\quad\text{(complementary slackness)},
\end{align}
where the derivative $\nabla_\alpha \Phi(f, \alpha^\star)$ is the displacement derivative \eqref{eq:displacement}, computed as
\begin{equation}
    \nabla_\alpha \Phi(f, \alpha^\star)(y) = 2 \exp(-\frac{f + T(-f, \alpha)}{2}).
\end{equation}
Therefore
\begin{align}
    \frac{f + T(-f, \alpha^\star)}{2} = - \log(-\frac{\nu}{2})\quad\text{on the support of $\alpha^\star$, and} \\
    \frac{f + T(-f, \alpha^\star)}{2} = - \log(-\frac{\mu(y) + \nu}{2}) \leq - \log(-\frac{\nu}{2})\quad\text{otherwise}.\label{eq:inequality_first_order}
\end{align}
Replacing in the definition $\Omega^\star(f) = - \log \Phi(f, \alpha^\star)$, and using the equality
\begin{equation}
    \Phi(f, \alpha) = \langle \alpha, \exp(-\frac{f + T(-f, \alpha)}{2}) \rangle
\end{equation}
we obtain
\begin{equation}
    - \log(-\frac{\nu}{2}) = \Omega^\star(f),
\end{equation}
and the first part of the lemma follows. Then, note that $T(f + c, \alpha) = T(f, \alpha) - c$ for all $c\in \RR$, $f \in \contspace$, $\alpha \in \probspace$. Removing $\Omega^\star(f)$ from both side of inequality \eqref{eq:inequality_first_order}, we obtain
\begin{equation}
    f - \Omega^\star(f) + T\big(-(f - \Omega^\star(f)), \nabla \Omega^\star(f)\big) \leq 0,
\end{equation}
with equality on the support of $\nabla \Omega^\star(f)$, which brings the second part of the lemma.
\end{proof}

\subsubsection{Proof of \autoref{lemma:inversion}}

\begin{proof}
    Let $\alpha \in \probspace$ and $f \triangleq \nabla \Omega(\alpha)$. From the optimality condition of Sinkhorn dual minimization~\eqref{eq:dual_kantorovich},
        \begin{equation}
            T(-f, \alpha) = - f,
        \end{equation}
        hence, $\alpha$ meets the sufficient conditions for optimality in \autoref{lemma:first_order}. Therefore $\nabla \Omega^\star(f) = \alpha$, $\Omega^\star(f) = 0$, and the first part of the lemma follows. To demonstrate the second part, we consider
        $f \in \cF$. There exists $\alpha \in \probspace$ such that $f = \nabla \Omega(\alpha)$, and thus
        \begin{equation}
            \nabla \Omega \circ \nabla \Omega^\star(f) = \nabla \Omega \circ \nabla \Omega^\star \circ \Omega(\alpha) = \nabla \Omega(\alpha) = f.
        \end{equation}
    The lemma follows.
    \end{proof}

\subsubsection{Extrapolation effect of $\nabla \Omega^\star$---Proof of \autoref{prop:extrapolation}}

We start by establishing some basic properties of the extrapolation operator.

\begin{lemma}[Properties of $f^E$]\label{lemma:properties}
    The following properties hold, for all $f \in \contspace$,
    \begin{enumerate}[label=\roman*.]
    \item The extrapolated potential $f^E$ verifies
    \begin{equation}
        f \leq f^E, \quad
        f_{| \supp \nabla \Omega^\star(f)} = f^E_{| \supp \nabla \Omega^\star(f)}.
    \end{equation}
    \item The extrapolation operator maintain the following values:
    \begin{equation}
        f^{EE} = f^E,\,\Omega^\star(f^E) = \Omega^\star(f), \,\nabla \Omega^\star(f^E) = \nabla \Omega^\star(f).
    \end{equation}
    \end{enumerate}
\end{lemma}

\begin{proof}We demonstrate (i), then (ii).

        i. Note that $T(f + c, \alpha) = T(f, \alpha) - c$ for all $c\in \RR$, $f \in \contspace$, $\alpha \in \probspace$. Removing $\Omega^\star(f)$ from both side of inequality \eqref{eq:inequality_first_order}, we obtain
        \begin{equation}
            f - \Omega^\star(f) + T\big(-(f - \Omega^\star(f)), \nabla \Omega^\star(f)\big) \leq 0,
        \end{equation}
        with equality on the support of $\nabla \Omega^\star(f)$.

        ii.  We set $\alpha = \nabla \Omega^\star(f)$. According to \autoref{lemma:first_order},
        for all $y \in \supp \alpha$, $f^E(y) = f(y)$ and
        \begin{equation}
            \frac{f^E(y) + T(-f^E, \alpha)(y)}{2} = 2 \Omega^\star(f).
        \end{equation}
        Furthermore, for all $y \in \cY$, $-f^E(y) \leq - f(y)$, and therefore, 
        as the soft C-transform operator is non-increasing with respect to $f$,
        \begin{equation}
            2 \Omega^\star(f) - f^E = T(-f, \nabla \Omega^\star(f)) \leq T(-f^E, \nabla \Omega^\star(f)),
        \end{equation}
        where the left equality stems from the definition of $f^E$. Therefore
        \begin{equation}
            \frac{f^E(y) + T(-f^E, \eta)(y)}{2} \leq 2 \Omega^\star(f),
        \end{equation}
        on all $\cY$, and we meet the sufficient condition of \autoref{lemma:first_order} for the optimality of $\eta$ in
        \begin{equation}
            \min_{\alpha \in \probspace} \Phi(f^E, \alpha).
        \end{equation}
        We thus have $\Omega^\star(f^E) = \Omega^\star(f)$, $\nabla \Omega^\star(f) = \nabla \Omega^\star(f^E)$. Therefore
        \begin{align}
            f^{EE} &= - T(-f^E, \nabla \Omega^\star(f^E)) + 2 \Omega^\star(f^E) \\
            &= - T(-f^E, \nabla \Omega^\star(f)) + 2 \Omega^\star(f) \\ 
            &= - T(-f, \nabla \Omega^\star(f)) + 2 \Omega^\star(f) = f^E,
        \end{align}
        where we have used on the third line the fact that the value of $T(f, \alpha)$ depends only on the values of $f$ on the support of $\alpha$. In our case, we have $f^E_{|\supp \nabla \Omega^\star(f)} = f_{|\supp \nabla \Omega^\star(f)}$, from \autoref{lemma:first_order}.
    The lemma follows.
\end{proof}

With \autoref{lemma:inversion} and \autoref{lemma:properties} at hand, we are now ready to prove~\autoref{prop:extrapolation}.

\begin{proof}
    We consider a function $f \in \contspace$. By construction of the extrapolation $f^E$,
        \begin{equation}
            g = f^E - \nabla \Omega^\star(f)
        \end{equation}
        is a negative symmetric Sinkhorn potentials, as $T(-g, \nabla \Omega^\star(f)) = -g$.
        Therefore, from~\autoref{lemma:inversion},
        \begin{align}
            \nabla \Omega \circ \nabla \Omega^\star(g) &= g \\
            \nabla \Omega \circ \nabla \Omega^\star(f^E) &= f^E - \nabla \Omega^\star(f) \\
            \nabla \Omega \circ \nabla \Omega^\star(f) &= f^E - \nabla \Omega^\star(f),
        \end{align}
        where the third equality stems from~\autoref{lemma:properties}, property (ii), and the second from~\eqref{eq:translation}.
\end{proof}

\subsection{Relation to Hausdorff divergence---Proofs of \autoref{prop:upper_bound} and \autoref{prop:bregman}}\label{app:statistic}

We now turn to proving~\autoref{prop:upper_bound} and~\autoref{prop:bregman},
that justifies the validity of the geometric logistic loss for a certain Bregman divergence, dubbed the asymmetric Hausdorff divergence.

\subsubsection{Proof of \autoref{prop:upper_bound}}

\begin{proof}
    Let $\alpha \in \probspace$ and $f \in \contspace$. By definition, the Hausdorff divergence $H = D_\Omega$ between $\alpha$ and $\nabla \Omega^\star(f)$ rewrites
    \begin{align}
        D_\Omega(\alpha | \nabla \Omega^\star(f)) &= \Omega(\alpha) - \Omega(\nabla \Omega^\star(f))
        - \langle \nabla \Omega \circ \nabla \Omega^\star(f), \alpha - \Omega^\star(f^E) \rangle \\
        &= \Omega(\alpha) + \langle f, \nabla \Omega^\star(f) \rangle - \Omega(\nabla \Omega^\star(f))
        - \langle f, \alpha \rangle + \langle f - \nabla \Omega \circ \nabla \Omega^\star(f), \alpha - \nabla \Omega^\star(f) \rangle \\
        &= \ell_\Omega(\alpha, f) + \langle f - \nabla \Omega \circ \nabla
        \Omega^\star(f), \alpha - \nabla \Omega^\star(f) \rangle.
    \end{align}
This decomposition is a generic way of decomposing a Bregman divergence into a
Fenchel-Young loss plus a perturbation term that depends on the ``projection'' $\nabla \Omega \circ \nabla \Omega^\star(f)$. In our case, thanks to \autoref{lemma:properties}, property (iv), this term rewrites
\begin{equation}
    \langle f - \nabla \Omega \circ \nabla \Omega^\star(f), \alpha - \nabla \Omega^\star(f) \rangle = \langle f - f^E, \alpha \rangle + \langle f - f^E, \nabla \Omega^\star(f) \rangle + \Omega^\star(f) \langle 1, \alpha - \nabla \Omega^\star(f) \rangle.
\end{equation}
The second term is null as a consequence of \autoref{lemma:properties}, while the third is null because $\alpha$ and $\nabla \Omega^\star(f)$ are both probability measures. The first one is null in case $\supp \nabla \Omega^\star(f) \in \supp \alpha$, in accordance to \autoref{lemma:properties}, property (i). The proposition follows from the fact that $f^E \geq f$ on the space $\cY$, according to the same property.
\end{proof}

\subsubsection{Proof of \autoref{prop:bregman}}

\begin{proof}
    As a consequence of \autoref{prop:upper_bound}, for any true and estimated distribution $\alpha, \hat \alpha \in \probspace$, we have
    \begin{equation}
        D_\Omega(\alpha | \hat \alpha) = D_\Omega(\alpha | \nabla \Omega^\star(\nabla \Omega(\alpha)))
        = \ell_\Omega(\alpha, \nabla \Omega(\alpha)) - \langle \alpha, {(\nabla \Omega(\alpha))}^E - \nabla \Omega(\alpha) \rangle,
    \end{equation}
    where the last term is null as $T(-\nabla \Omega(\alpha), \alpha) = -\nabla \Omega(\alpha)$ and $\Omega^\star(\nabla \Omega(\alpha)) = 0$ from \autoref{lemma:inversion}.
    Therefore
    \begin{equation}
        D_\Omega(\alpha | \hat \alpha) = \ell_\Omega(\alpha, \nabla \Omega(\alpha)).
    \end{equation}
    The equality of risks and the connection between minimizers immediately follows. To establish the Fisher consistency of the g-FY loss with respect to the Hausdorff divergence, note that, from \autoref{prop:upper_bound}, we have, for all $\hat f: \cX \to \contspace$, for all $x \in \cX$, 
    $\alpha \in \probspace$,
    \begin{equation}
        D_\Omega(\alpha | \nabla \Omega^\star(\hat f(x)) \leq \ell_\Omega(\alpha, \hat f(x)).
    \end{equation}
    Taking the expectation with respect to the data distribution $\cD$, we obtain
    \begin{equation}
        \cE(\nabla \Omega^\star \circ \hat f) \leq \cR(\hat f),
    \end{equation}
    and the proposition follows.
\end{proof}

\section{Further experiments and details}\label{app:exp_details}

\subsection{Variational auto-encoders}

\paragraph{High definition experiment.}As a complementary experiment, we generate a dataset of cat doodles from the Google QuickDraw dataset, with a line width of one pixel. We test the g-softmax link function
and the geometric Fenchel-Young loss functions to train a VAE with a DC-GAN architecure~\cite{radford2015unsupervised}. We reuse the architecture of the authors, using the  discriminator as an encoder, with a final layer with a size of output twice the size of the latent dimension, to model the mean and variance of the latent encoding, and the generator as a decoder. Similarly to
the experiment in the main text, we observe that the generated samples and the reconstructions are more concentrated on thin measures.

\paragraph{MNIST.}We display a visualization of generates images and
reconstruction of test image in \autoref{fig:mnist}. The output distributions
are well concentrated, despite the low resolution of the dataset.

\paragraph{Architecture}Our multi-layer perceptron is simple: encoder and decoder are two layer MLP with 400 hidden units and ReLU activation.

\paragraph{Hyperparameters.}We use a latent size of 100 in the experiment on QuickDraw 28x28, and 256 for the high resolution experiment. We set the KL weight to $1$, and rescale the KL loss with a factor $h \times w$, to make its gradient of the same order as the one computed with separated binary cross entropy. We use $\sigma = 2$ as the scaling parameter of the Euclidean cost function.

\begin{figure}
    \centering
    \includegraphics[width=.3\textwidth]{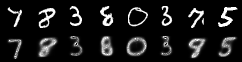}
    \includegraphics[width=.3\textwidth]{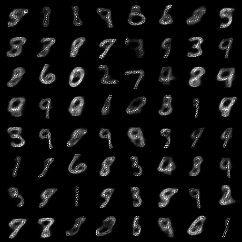}
    \caption{Examples of generated images and reconstruction of test images with an MLP VAE on MNIST dataset.}\label{fig:mnist}
\end{figure}

\begin{figure}
    \centering
    \includegraphics[width=\textwidth]{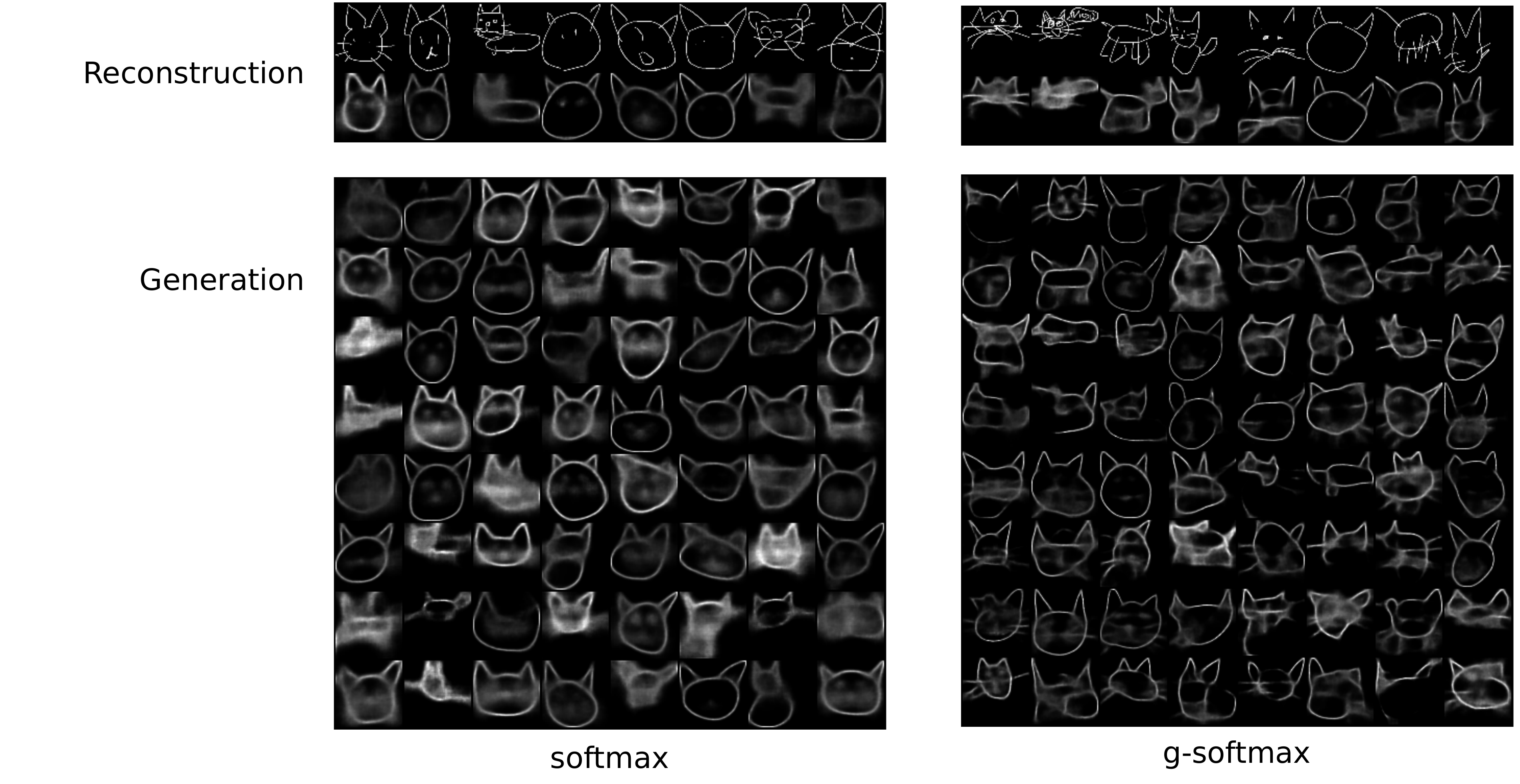}
    \caption{Examples of generated images and reconstruction of test images with a VAE-DC-GAN and a geometric softmax last layer. The generated images are sharper than when using a standard softmax layer and a KL divergence training.}\label{fig:big_vae}
\end{figure}


\end{document}